\documentclass{article}

\usepackage{microtype}
\usepackage{graphicx}
\usepackage{subfigure}
\usepackage{booktabs} % for professional tables

\usepackage{hyperref}

\usepackage[accepted]{icml2020}

\usepackage{amsmath}
\usepackage{amssymb}
\usepackage{amsthm}
\usepackage{algorithm}
\usepackage{algorithmic}

\newtheorem{theorem}{Theorem}[section]
\newtheorem{corollary}{Corollary}[theorem]
\newtheorem{lemma}[theorem]{Lemma}
\newtheorem{definition}{Definition}[section]
\newtheorem{fact}{Fact}[section]

\def\ourboundmath{(3 + 2\sqrt{\gamma})\kappa + 2\gamma\sqrt{C\log(1/\gamma)}}

\icmltitlerunning{Designing Differentially Private Estimators in High Dimensions}

\begin{document}

\twocolumn[
\icmltitle{Designing Differentially Private Estimators in High Dimensions}



\icmlsetsymbol{equal}{*}

\begin{icmlauthorlist}
\icmlauthor{Aditya Dhar}{equal,h}
\icmlauthor{Jason Huang}{equal,h}
\end{icmlauthorlist}

\icmlaffiliation{h}{Harvard University, Cambridge, MA, United States}

\icmlcorrespondingauthor{Aditya Dhar}{adhar@college.harvard.edu}
\icmlcorrespondingauthor{Jason Huang}{jasonhuang@college.harvard.edu}

\icmlkeywords{robustness, machine learning, ICML, differential privacy}

\vskip 0.3in
]



\printAffiliationsAndNotice{\icmlEqualContribution} % otherwise use the standard text.

\begin{abstract}
We study differentially private mean estimation in a high-dimensional setting. Existing differential privacy techniques applied to large dimensions lead to computationally intractable problems or estimators with excessive privacy loss. Recent work in high-dimensional robust statistics has identified computationally tractable mean estimation algorithms with asymptotic dimension-independent error guarantees. We incorporate these results to develop a strict bound on the global sensitivity of the robust mean estimator. This yields a computationally tractable algorithm for differentially private mean estimation in high dimensions with dimension-independent privacy loss. Finally, we show on synthetic data that our algorithm significantly outperforms classic differential privacy methods, overcoming barriers to high-dimensional differential privacy.
\end{abstract}

\section{Introduction}

\subsection{Background and Problem Statement}

An algorithm for releasing output from a database satisfies $\epsilon$-differential privacy if adding, removing, or changing a record in the database does not result in a significant change in the output of the algorithm, where the allowable change is determined by $\epsilon$. The adoption of differentially private algorithms by the US Census Bureau highlights the importance of differential privacy in, among other high-level applications, ensuring the privacy of data release. However, existing methods of differential privacy become complicated when dealing with the dimensionality of the dataset, which often creates excessive dimension-dependent error or makes differential privacy algorithms computationally intractable \cite{bayes, vldb}.

It is well known that the Gaussian mechanism guarantees differential privacy. So, if we can derive the sensitivity of the function to which we want to add noise in computationally tractable time, we then can make said the output of that function differentially private. Formally, we then consider the following problem for mean estimation: given a function, $f$, and a set of points $x$, find a strict, dimension-independent upper bound on the following objective:
\[\max_{\{x, x': d(x, x') = 1\}}||f(x) - f(x')||\]

The concept of sensitivity in differential privacy has strong similarities to other concepts of statistical stability, notably in robustness. This connection was first investigated in-depth by Dwork and Lei \cite{dworkrobust}. They proposed a framework known as \texttt{Propose-Test-Release} (\texttt{PTR}) for making robust statistics differentially private. The \texttt{PTR} approach was inspired by the fact that robust statistics often have bounded influence functions. \texttt{PTR} enables the use of local sensitivity, which is typically much smaller than global sensitivity but not differentially private when applied naively, in calibrating noise \cite{smooth}. For instance, under i.i.d. data draws, the sample interquartile range should be $O(1/\sqrt{n})$ away from the distribution interquartile range, leading the authors to use this as the scale for the proposed sensitivity.

Recent work on robust statistics have focused in the high-dimensional case, where many historical approaches either yield dimension-dependent error or are computationally intractable. A common theme in this recent literature is the development of certificates for robustness, which has important implications for differential privacy query sensitivities. However, simply using \texttt{PTR} with these asymptotic error guarantees presents the challenge of unbounded and potentially intractable testing steps. This paper therefore adapts and augments these error guarantees to bound the global sensitivity for use in designing a differentially private high-dimensional mean estimator. We extend the work of DKKLMS in \cite{hightractable, practical}, who find a computationally tractable robust mean estimation algorithm with dimension-independent error.

We advance the state-of-the-art in several aspects:
\begin{itemize}
    \item We derive an upper bound on the global sensitivity of mean estimators by converting the asymptotic bounds found in \cite{practical} to strict bounds. 
    \item We design a differentially private algorithm for the robust mean estimation in the presence of adversarial corruption, as well as a differentially private algorithm absent corruption.
    \item We show that that our algorithms do not incur dimension-dependent privacy loss, nor do they make assumptions that the population mean lie in some known bounded interval, as many popular techniques like the Winsorized mean currently do.
    \item We show that this has the added benefit of not requiring additional computational complexity to compute the bound once mean estimation is done, so the algorithm to release our statistic is computationally tractable.
    \item We empirically compare our algorithms with conventional attempts at guaranteeing privacy over large datasets, and show that our estimators significantly outperforms existing methods for high-dimensional data.
        %
    %To the best of our knowledge, our algorithm for mean estimation is the first differentially private estimator of the mean for which one can prove good statistical guarantees without any boundlessness assumptions on the data and without assuming that the population mean lie in some known bounded interval. In particular, our construction does not rely on any truncation of the data and gives optimal, sub-Gaussian deviations when the data are heavy tailed
    % What does this mean? => \item Our algorithm for mean estimation thus provides good statistical guarantees without any boundlessness assumptions on the data or on the population mean.
\end{itemize}
 
\subsection{Related and Prior Work}

\paragraph{Robust Statistics} Robustness is a field with a long history, pioneered decades ago by John Tukey \cite{tukey}. The field has come a long way since then \cite{huber} and continues to grow. Recent focus has turned to robustness for high-dimensions given the feature richness and scale of modern data analysis tasks. While some techniques fail to hold in the high-dimensional case, recent advancements still provide good results for estimation in large datasets. DKKLMS make use of filtering methods to prune corrupted points from a set of data, showing that filtering methods requires computation of the largest eigenvalue of a covariance matrix, yielding a computationally tractable algorithm with dimension-independent guarantees. Lai, Rao, and Vempala similarly use spectral methods for agnostic mean estimation \cite{agnostic}. Hopkins and others achieve similar robustness results for both Gaussian and heavy-tailed distributions, relying on \texttt{Sum-of-Squares} proofs and semi definite programming instead of filtering methods \cite{mixture, heavy, hard}. This field continues to evolve, with recent works applying gradient estimation and descent to the \texttt{Sum-of-Squares} hierarchy to yield similar error guarantees with faster runtime \cite{fast}.

\paragraph{Differential Privacy} Private data analysis enabled by differential privacy is constantly advancing with new techniques and applications \cite{dpbook}. Differential privacy in robust estimators was first detailed by Dwork and Lei \cite{dworkrobust}, who outlined the \texttt{Propose-Test-Release} framework for mapping robustness results to differential privacy. Work since then has included results on differentially private M-estimators \cite{mestimators}. A recent paper also generalizes \texttt{Propose-Test-Release} by applying finite sample breakdown points, a recurring aspect of robust statistics, over discretized bins, increasing the probability of a reply by the \texttt{Propose-Test-Release} algorithm and decreasing the estimation error \cite{fast}; thus yielding a high probability bound for differentially private mean estimation without requiring assumptions on the boundedness of the data. The concept of sensitivity has also been studied in great detail and extended by works like \cite{smooth}, which define concepts like local and smooth sensitivity for achieving higher utility while still maintaining privacy. In high dimensions, work centers around efforts to privately learn multivariate Gaussians by 'clamping' the sensitivity of the empirical covariance matrix \cite{highdim}.

\section{Preliminaries}

\subsection{Differential Privacy}

A line of work known as differential privacy has emerged for providing mathematical guarantees on privacy loss beginning from \cite{calibrating}.

\begin{definition}[Differential Privacy \cite{dpbook}] A randomized function $\mathcal{M}$ is considered to give $(\epsilon, \delta)$-differential privacy if for all adjacent data sets $x, x'$ and all $S\subseteq Range(\mathcal{M})$:
\[P[\mathcal{M}(x)\in S] \leq \exp(\epsilon)P[\mathcal{M}(x')\in S]+\delta\]
\end{definition}
Intuitively, this is saying that data sets differing by a single individual should yield query results that differed in probability by a multiplicative factor of at most $\exp(\epsilon)$ and by an additive factor of at most $\delta$. When $\delta=0$, this is simply referred to as $\epsilon$-differential privacy or pure differential privacy.

In order to provide such guarantees, there must be some restriction for how the function differs between these adjacent data sets. This is known as sensitivity.
\begin{definition}[Global Sensitivity \cite{smooth}]
    The global $\ell_1$-sensitivity of a function $f$ is:
    \[\Delta f = \underset{x, x':d(x, x')=1}{\max}||f(x) - f(x')||\]
\end{definition}
Different techniques have been developed for performing differentially private data analysis. A common approach is the Gaussian mechanism, which adds noise drawn from a Gaussian distribution with scale parameter $b = 2 \ln(1.25/\delta)(\Delta f)^2/\epsilon^2$ to the result of the query function $f(x)$.
\begin{fact}
The Gaussian mechanism guarantees $\epsilon$-differential privacy \cite{dpbook}.
\end{fact}

\subsection{Robust High-Dimensional Mean Estimation by Filtering}

DKKLMS study the task of robust mean estimation specifically for high-dimensions \cite{hightractable, practical}. The specific adversarial framework they use, known as $\gamma$-corruption, is powerful and can easily be generalized to other models.\footnote{DKKLMS use $\epsilon$ as the corruption parameter, but it has been replaced here with $\gamma$ to avoid confusion with the $\epsilon$ privacy loss parameter of differential privacy}
\begin{definition}[$\gamma$-corruption \cite{practical}]
    Given $\gamma > 0$ and a set of samples of size $m$, the samples are $\gamma$-corrupted if an adversary is allowed to inspect the samples, remove $m' \sim \text{Bin}(\gamma, m)$ of them, and replace them with arbitrary points.
\end{definition}

The authors proceed to make a key observation involving a \emph{certificate of robustness}. Specifically, bounding the spectral norm of the empirical covariance matrix will provide dimension-independent error guarantees for the mean estimate. They use this certificate to design a novel algorithm with dimension-independent error guarantees.

\begin{theorem}\label{t2.1}
    Let $G$ be a sub-Gaussian distribution on $\mathbb{R}^d$ with parameter $\nu = \Theta(1)$, mean $\mu^G$, covariance matrix $I$, and $\gamma > 0$. Let $S$ be a $\gamma$-corrupted set of samples from $G$ of size $\Omega((d/\gamma^2)\text{ poly }\log(d/\gamma))$. There exists an efficient algorithm that, on input $S$ and $\gamma > 0$, returns a mean vector $\hat{\mu}$ so that with probability at least 9/10 we have $||\hat{\mu} - \mu^G||_2 = O(\gamma\sqrt{\log(1/\gamma)})$ \cite{practical}.
\end{theorem}

Note that we assume the uncontaminated data generating process follows a sub-Gaussian distribution, since DKKLMS provide an algorithm satisfying Theorem \ref{t2.1} for the sub-Gaussian case, called \texttt{Filter-Gaussian-Unknown-Mean} \cite{hightractable}. DKKLMS also propose a related algorithm for heavy-tailed distributions that requires bounded second moments, a more general condition on large real-world datasets that extensions on our work could incorporate. WHile our paper is based on the prior sub-Gaussian algorithm, which limits the real-world applicability of our research in that we can only apply it to data generated by gaussians, our paper stands as a proof of concept for better connections between robustness and privacy, and can easily be generalized as such to the heavy-tailed algorithm.

Each recursive call of the \texttt{Filter-Gaussian-Unknown-Mean} algorithm computes the spectral norm of the empirical covariance and compares the result against a threshold. If $||\Sigma||_2 \leq \text{Thresh}(\gamma)$, where $\text{Thresh}(\gamma) = C\gamma\log(1/\gamma) = O(\gamma\log(1/\gamma))$ for some constant $C$, then by the certificate of robustness, the error is appropriately bounded and the empirical mean can be returned. Otherwise, the algorithm will project the data set onto a particular direction and remove data points that are far from the current empirical mean. The algorithm will repeat this procedure on this new filtered set.

\section{Robustness Error Guarantees}\label{3}

In this section, we construct a novel and strict bound on the global sensitivity of the robust estimation procedure, enabling the application of differential privacy. We follow the procedure used to establish Theorem \ref{t2.1} in \cite{hightractable, practical}, and apply various asymptotic inequalities to generate said bound.

DKKLMS begin by drawing from $G$, an identity covariance Gaussian in $d$ dimensions with mean $\mu^G$. If $S$ is the original, corrupted dataset, there exist disjoint multisets $L, E$ where $L \subset S$ such that $S' = (S\backslash L) \cup E$ represents the final dataset produced by the iterative \texttt{Filter-Gaussian-Unknown-Mean} algorithm, such that $\mu^{S'} = \hat{\mu}$ is the mean estimate ultimately returned by the estimation procedure. The mean of each multiset $\mathcal{F}\in \{S, S', L, E\}$ will be represented as $\mu^{\mathcal{F}}$.

The goal is to derive a bound on $||\mu^{S'} - \mu^G||_2$, which represents the error between the empirical mean of the final filtered data samples and the true distribution mean. To begin, the following is true by definition.
\[|S'| (\mu^{S'}-\mu^G) = |S|(\mu^S - \mu^G) - |L|(\mu^L - \mu^G) + |E| (\mu^E -\mu^G)\]
\begin{equation}\label{eq:1} \mu^{S'} - \mu^G = \dfrac{|S|}{|S'|}(\mu^S - \mu^G) - \dfrac{|L|}{|S'|}(\mu^L - \mu^G) + \dfrac{|E|}{|S'|} (\mu^E -\mu^G)\end{equation}
It then suffices to bound each of these three terms using the following three lemmas. The full proofs for these lemmas is given in the Appendix. 
\begin{lemma}\label{lemma1}
We have that $(|S|/|S'|)||\mu^S - \mu^G||_2 \leq \frac{\gamma}{1-2\gamma}$.
\end{lemma}

\begin{lemma}\label{lemma2}
We have that $(|L|/|S'|)||\mu^L - \mu^G||_2 \leq \frac{\sqrt{2}\gamma+\sqrt{2\gamma}}{1-2\gamma}$.
\end{lemma}

\begin{lemma}\label{lemma3} Let $\kappa = \frac{\gamma}{1-2\gamma} + \frac{\sqrt{2}\gamma+\sqrt{2\gamma}}{1-2\gamma}$ and $\lambda^*$ represent the largest eigenvalue of the empirical covariance matrix. Then we have $(|E|/|S'|)||\mu^E - \mu^G||_2 \leq (2 + 2\sqrt{\gamma})\kappa + 2\sqrt{\gamma\lambda^*}$.
\end{lemma}
Using these lemmas, the desired expression can now be bounded as stated in the following theorem.

\begin{theorem}\label{thm}
     Let $\kappa = \frac{\gamma}{1-2\gamma} + \frac{\sqrt{2}\gamma+\sqrt{2\gamma}}{1-2\gamma}$ and $\lambda^*$ represent the largest eigenvalue of the empirical covariance matrix. Then we have $||\mu^{S'} - \mu^G||_2 \leq (3 + 2\sqrt{\gamma})\kappa + 2\sqrt{\gamma\lambda^*}$
\end{theorem}
\begin{proof}
    This follows by bounding Equation \ref{eq:1} using Lemmas \ref{lemma1}, \ref{lemma2}, and \ref{lemma3}, and an application of the triangle inequality.
\end{proof}
\begin{corollary}\label{3.4.1} Let $\hat{\mu}$ represent the final output of the recursive filtering process. Then we have that:
\[||\hat{\mu} - \mu^G||_2 \leq \ourboundmath\]
\end{corollary}
\begin{proof}
    Recall that $S'$ is analogous to the well-behaved multiset after filtering. The termination condition is the case of the small spectral norm, meaning that $\lambda^*$ is bounded by $C\gamma\log(1/\gamma)$.
\end{proof}

\begin{corollary} Let $\hat{\mu}^*$ represent the final output of the recursive filtering process on a dataset with one corrupted point. Then we have that:
\[||\hat{\mu} - \mu^G||_2 \leq \left(3+\frac{2}{\sqrt{n}}\right)\left(\dfrac{1+\sqrt{2}+\sqrt{2n}}{n-2}\right) + \dfrac{2\sqrt{C\log(n)}}{n}\]
\end{corollary}
\begin{proof}
    This follows from $\gamma = \frac{1}{n}$ in Corollary \ref{3.4.1}.
\end{proof}

\section{Differentially Private Robust Mean Estimation}

Our contribution is the following differentially private algorithm for robust mean estimation in high dimensions. It uses the work by DKKLMS \cite{practical} to calculate the true robust result. The error guarantees have been adapted in the proofs above to provide an upper bound on global sensitivity, such that noise can then be appropriately added using the mean-zero Gaussian noise mechanism.
\begin{algorithm}
\caption{\texttt{DP-Robust-Mean-Estimation}}\label{alg}
\begin{algorithmic}
    \REQUIRE data set $S$, corruption level $\gamma$, confidence level $\tau$, privacy loss $\epsilon$, threshold factor $C$
    \STATE $\hat{\mu}\leftarrow \texttt{Filter-Gaussian-Unknown-Mean}(S, \gamma, \tau)$
    \STATE $\hat{\mu}_{DP} \leftarrow \hat{\mu}$ +\STATE $\mathcal{N}\left[0, \dfrac{8\ln(\frac{1.25}{\tau})}{\epsilon^2}\left( \ourboundmath\right)^2 \right]$
    \RETURN $\hat{\mu}_{DP}$
\end{algorithmic}
\end{algorithm}

\begin{algorithm}
\caption{\texttt{DP-Mean-Estimation}}\label{alg2}
\begin{algorithmic}
    \REQUIRE data set $S$, confidence level $\tau'$, privacy loss $\epsilon'$, threshold factor $C$
    %\STATE $s \sim \text{Unif}(\min(S), \max(S))$
    %\STATE $S' \leftarrow S \cup s$
    \STATE $\hat{\mu}\leftarrow \texttt{Filter-Gaussian-Unknown-Mean}(S, \frac{1}{n}, \tau')$
    \STATE $\hat{\mu}_{DP} \leftarrow \hat{\mu}$ + \STATE $\mathcal{N}\left[0, \dfrac{8\ln(\frac{1.25}{\tau'})}{\epsilon'^2}\left(\frac{\left(3+2/\sqrt{n}\right)\left(1+\sqrt{2}+\sqrt{2n}\right)}{n-2} + \frac{2\sqrt{C\log(n)}}{n}\right)^2 \right]$
    \RETURN $\hat{\mu}_{DP}$
\end{algorithmic}
\end{algorithm}

\begin{theorem}
    The \emph{\texttt{DP-Robust-Mean-Estimation}} algorithm returns a robust mean in polynomial time with $(\epsilon, \tau)$-differential privacy. The  \emph{\texttt{DP-Mean-Estimation}} returns a $(\epsilon',\tau')-$ differentially private mean.
\end{theorem}
\begin{proof}
    The correctness and polynomial runtime of robust mean estimation is given by \cite{practical}. The only additional step is adding Gaussian noise, which runs linear time with respect to the number of dimensions.
    
    The certificate of robustness bounds the sensitivity of the robust mean estimator $f$, which is computed using the \texttt{Filter-Gaussian-Unknown-Mean} algorithm of \cite{practical}. From Theorem \ref{thm}, then $||f(S) - \mu||_2 \leq \ourboundmath$. Therefore, for any adjacent data sets $S, S'$, by the triangle inequality $||f(S) - f(S')||_2 \leq ||f(S) - \mu||_2 + ||f(S') - \mu||_2 \leq 2\left(\ourboundmath\right)$. The sensitivity of the robust mean estimation task is $\Delta f = 2\left(\ourboundmath\right)$. Certain elements of the proof hold with probability $1-\tau$, which provides the additive privacy loss term. Since Gaussian noise is added to the resulting estimate with scale of $2\ln(1.25/\delta)(\Delta f)^2 / \epsilon^2$, it follows by the correctness of the Gaussian mechanism that \texttt{DP-Robust-Mean-Estimation} is $(\epsilon, \tau)$-differentially private.
\end{proof}

\section{Empirical Results}

% \begin{figure}[h]
% \begin{center}
%     %\includegraphics[width=0.3\textwidth]{figs/100.jpg}
%     \includegraphics[width=0.32\textwidth]{figs/1000.jpg}
%     \includegraphics[width=0.32\textwidth]{figs/10000.jpg}
%     \includegraphics[width=0.32\textwidth]{figs/100000.jpg}
% \end{center}
% \caption{Excess $\ell_2$ error between differentially private mean estimation using filtering and Winsorization with varying (a) $n=1000$, (b) $n=10000$, and (c) $n=100000$.}% $n = \{1000, 10000, 100000\}$}
% \end{figure}

\begin{figure}[ht]
% \vskip 0.2in
\begin{center}
\centerline{\includegraphics[width=\columnwidth]{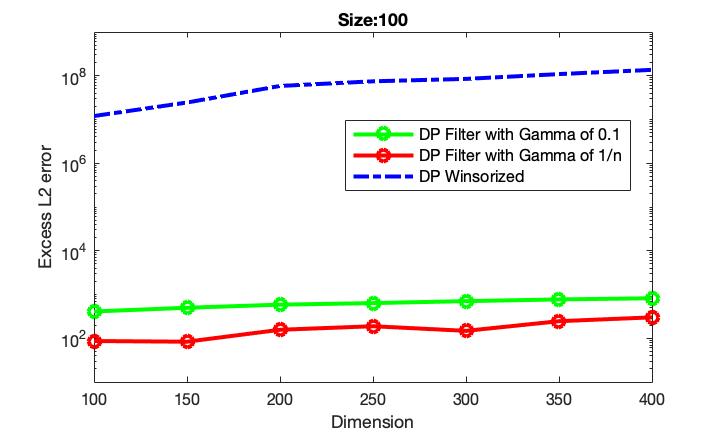}}
\caption{Excess $\ell_2$ error between the differentially private mean using filtering and the differentially private Winsorized mean with $n = 100$.}\label{fig1}
\end{center}
\vskip -0.3in
\end{figure}

\begin{figure}[ht]
% \vskip 0.2in
\begin{center}
\centerline{\includegraphics[width=\columnwidth]{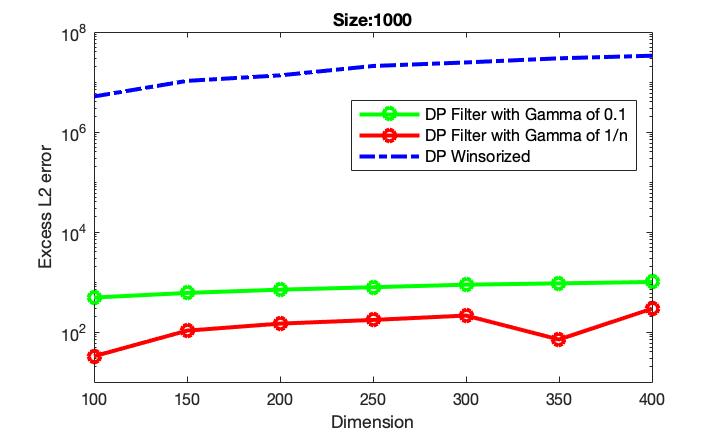}}
\caption{Replication of Figure \ref{fig1} for $n = 1000$.}\label{1b}
\end{center}
\vskip -0.3in
\end{figure}

\begin{figure}[ht]
% \vskip 0.2in
\begin{center}
\centerline{\includegraphics[width=\columnwidth]{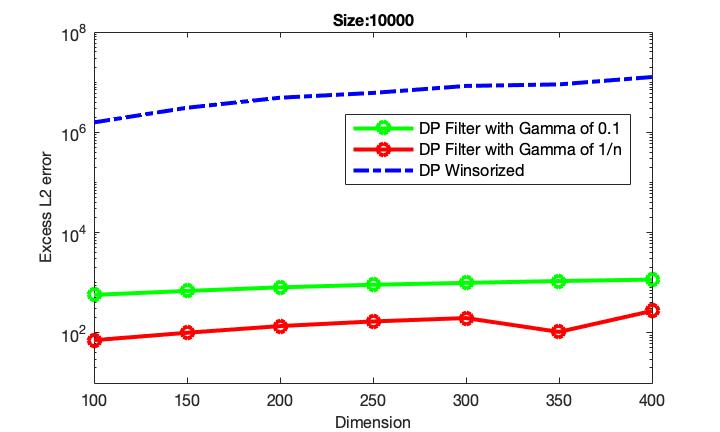}}
\caption{Replication of Figure \ref{fig1} for $n = 10000$.}\label{1c}
\end{center}
\vskip -0.5in
\end{figure}

We performed an empirical evaluation of the above algorithm on synthetic data, using the implementation of \cite{practical} for the robust mean estimation algorithm. The focus of this evaluation was on statistical accuracy, and we have a series of results for varying parameters $n$, $\gamma$, where $n$ refers to the number of overall data points and $\gamma$ is, as identified earlier, the level of ‘corruption’ that the robust mean estimation algorithm eliminates. We include $\gamma$ to parametrize the bound on sensitivity, and use $\gamma = 0.1$ for Algorithm \ref{alg} and $\gamma = \frac{1}{n}$ for Algorithm \ref{alg2}.

In our evaluation, we compare the error loss generated by our differentially private algorithm to the error loss generated by the differentially private Winsorized mean. Winsorization is a statistical technique that clamps data points to a specified interquantile interval by rounding the bottom $\alpha n$ data points up to the $\alpha$-percentile and rounding the top $\alpha n$ data points down to the $(1-\alpha)$-percentile. This restricts the global sensitivity of the mean, leading to the design of differentially private algorithms as in \cite{estimation}, the reference algorithm used in our evaluations. In all of these experiments, we set the privacy loss parameter $\epsilon=1$. Code of our implementation is available at \url{https://github.com/TurboFreeze/dp-robust-filter}.

We make several observations of note. First, both algorithms outperform the differentially private Winsorized algorithm in medium to large datasets\footnote{Our algorithms also outperform the differentially private Winsorized algorithm in high dimensions for larger $n$, including $n = 100,000$. For brevity, we do not include these results, but they are easily reproducible with our implementation code}; the scale of error shown in Figures \ref{fig1}, \ref{1b}, and \ref{1c} is several orders of magnitude lower in our results than the differentially private Winsorized mean. In modern data analysis, data sets are often of the range that we study here, in the thousands or tens of thousands of observations. This is especially true in the social sciences, where data privacy is generally most relevant. As such, improving the performance of estimators on this range of data is an important step towards advancing privacy in these areas.

Second, Algorithm \ref{alg} has $\ell_2$-error constant as $n$ decreases. This is obviously by construction, the bounds in Corollary \ref{3.4.1} are all $n$-independent even if dimension-dependent error is still generated. We also identify that Algorithm \ref{alg2} has $\ell_2$-error slightly decreasing as $n$ increases and outperforms Algorithm \ref{alg} and the Winsorized mean as it has less corruption that generates adversarial noise. 

While intuitive, these are useful result to note: the performance of differentially private Winsorized means sharply deteriorates with lower $n$, and makes our algorithm a far preferable solution, especially when dealing with medium to large as opposed to very large datasets. We further identify that, while differentially private Winsorized means can initially adjust for hihg-dimensionality because the computation of the median includes a $\frac{1}{n}$ term, as $\frac{d}{n}$ increases, the effect of dimensionality outweighs the positive effects of larger datasets; and the positive effects of larger datasets do not even exist on the medium to high scales we consider. This explains the steepness in the curves marked by our graphs; we show above that the differentially private algorithms increase in error much faster than both our algorithms, especially our robust differentially private algorithm.

Finally, as a quick comparison between the two algorithms, we utilize a rather high $\gamma$ value of 0.1 in evaluating Algorithm \ref{alg}. This is of particular significance: paring down the $\gamma$ value intuitively makes the bound stronger, both by definition of the strict bound that we have shown, but also by intuition: a robust estimator that is robust to less corruption of the data approaches a differentially private estimator that is only required to be 'robust' to one adjusted data point; said estimator constitutes Algorithm \ref{alg2}. In and of itself, however, the fact that we can still achieve starkly better rates than Winsorized means on our loose sensitivity bound with high $\gamma$ suggests the power of the algorithm.

\section{Discussion}

Conventional differential privacy techniques generally have privacy loss parameters that scale with the number of dimensions. Naive truncation, as in the cases of $\alpha$-trimmed or Winsorized means, will thus have dimension-dependent privacy loss. As a result, existing algorithms require databases to be bounded or assume that the parameter in question lies in a bounded and known interval. These are problematic assumptions to require, because it means we cannot construct differentially private estimators over commonly used distributions, including normal and t-distributions. Moreover, this is often infeasible, especially in high dimensions. Having the user define bounds in which the population mean lays for $\alpha-$trimmed and Winsorized means to work requires the user defining the bounds for the large number of dimensions in this setting, which is impractical. Our result makes a significant contribution in this regard: by providing strict dimension-independent error guarantees without relying on any metadata, we can achieve accurate results on large-dimensional datasets without requiring boundedness constraints. This also allows the algorithm to function unsupervised without user-provided bounds in each dimension. The significance of the work done by DKKLMS in  \cite{hightractable, practical} and others like Hopkins and Li \cite{hard, mixture}, which achieve dimension-independent errors, thus yields an effective filtering procedure that underlies this paper. In and of itself, generating dimension-independent privacy loss thus represents a significant step forward in applying differential privacy to high-dimensional settings.

The differentially private algorithm proposed in this paper uses the same effective filtering procedure proposed by DKKLMS, but still requires adding noise in each dimension in order to achieve differential privacy. We make two remarks with regard to this. First, we opt into the use of global sensitivity to calibrate noise as opposed to the local sensitivity across each dimension. There is recent work in this direction \cite{highdim}. We choose to use global sensitivity because our algorithm gives a computationally tractable bound on said sensitivity; local sensitivity remains a direction for future work. Second, privacy can be lost through any dimension; this proposal should be fairly intuitive. Therefore, differential privacy would inevitably require noise in each dimension. As a result, even if we can entirely eliminate dimensionality as a concern in privacy loss, differentially privacy algorithms would inevitably introduce some dimension-dependent error. This remains a significant issue inherent to any high-dimensional privacy problems, not just robust mean estimation as studied in this paper, and could pose a challenge for imposing differential privacy in modern machine learning or institutional data release for high-dimensional datasets. Our empirical results suggest that our algorithm is a good solution to the conventional 'curse of dimensionality', where increasing $\frac{d}{n}$ ratios result in algorithms lacking the information to compute accurate estimates in higher dimensions. We note that by bounding privacy loss, even though we cannot achieve fully dimension-independent error, we can sharply minimize error that would otherwise grow exponentially in $d$.

Finally, we note the importance of spectral methods as used in prior work. These methods remain a powerful mechanism that are gaining traction for designing and validating the robustness of statistics. While the work in this paper is only for one specific technique, it can be adapted to a number of different settings as discussed in related work. The work by Diakonikolas et al., for example, gives the case of bounded second moments in addition to the general sub-Gaussian case studied here; the work by Hopkins, in analyzing both Gaussian and sub-Gaussian deviations, applies here as well.

\section{Conclusion}

In this paper, we investigate the connection between two notions of statistical stability: query sensitivity and robustness, which are critical in designing privacy-preserving and secure algorithms, respectively. In particular, our primary contribution is leveraging stability that inherently arises from recent developments in high-dimensional robustness for use in differential privacy. Furthermore, we show that our approach increases utility in privacy-preserving data analysis on medium- to large-sized data sets, hopefully encouraging the use of differential privacy through this improved utility. 

The work here thus presents a novel algorithm that reduces privacy loss and improves tractability in high dimensional settings, marking a large step forward for analyzing differential privacy in such cases. By concretizing asymptotic bounds yielded from \cite{practical}, we develop a dimension-independent guarantee for privacy loss, which can be tractably computed in high dimensions. Our resulting empirical work additionally identifies a problem which appears to hold true for differential privacy in the high dimensional setting: because privacy leakage can occur from any dimension, the generation of a differentially private statistic results in dimension-dependent error even when privacy loss is dimension-independent. We note then that if dimensionality is a permanent barrier to differential privacy in large datasets, our algorithm provides the best possible error guarantee by eliminating dimension-dependent privacy loss and having dimensionality impact only the generation of Gaussian noise. This paper thus provides important results to understanding differential privacy in high dimensions.

\section{Acknowledgements}
We would like to thank Cynthia Dwork for her helpful instruction, insightful advice, and careful guidance; and Gautam Kamath for taking the time to read our paper and provide his thoughtful criticisms.

% \newpage

\bibliographystyle{plainnat}
\bibliography{bib.bib}

\section{Appendix}
\label{appa}

We begin by defining $(\gamma, \tau)$-good multisets, which is used to establish Theorem 2.1 in \cite{hightractable, practical}.

\begin{definition}[$(\gamma, \tau)$-good multisets \cite{hightractable}]\label{gooddef}
    Let $G$ be an identity covariance Gaussian in $d$ dimensions with mean $\mu^G$, and $\gamma, \tau > 0$. A multiset $S$ is $(\gamma, \tau)$-good with respect to $G$ if:
    \begin{enumerate}
        \item For all $x\in S$, $||x- \mu^G||_2\leq O(\sqrt{d\log(|S|/\tau)})$
        \item For every affine function $L:\mathbb{R}^d\to\mathbb{R}$ such that $L(x)=v\cdot(x-\mu^G)-T$, $||v||_2=1$, we have that $|P_{X\in_u S}[L(X)\geq 0] -P_{X\in G}[L(X)\geq 0]| \leq\frac{\gamma}{T^2\log(d\log(d/\gamma\tau))}$
        \item $||\mu^S-\mu^G||_2\leq \gamma$
        \item $||M_S-I||_2\leq \gamma$
    \end{enumerate}
\end{definition}

\begin{lemma}\label{good}
Let $G$ be an identity covariance Gaussian with $\gamma, \tau > 0$. If the multiset $S$ is obtained by taking $\Omega((d/\gamma^2)\text{ poly }\log(d/\gamma\tau))$ independent samples from $G$, then $S$ is $(\gamma, \tau)$-good with respect to $G$ with probability at least $1-\tau$.
\end{lemma}
This lemma follows from a proof given in Appendix B of \cite{hightractable}. 

Defining disjoint multisets $L, E$ as in Section \ref{3}, we have $L \subset S$ and $S' = (S\backslash L) \cup E$ representing the final data set produced by the iterative  \texttt{Filter-Gaussian-Unknown-Mean}. We denote the mean of each multiset $\mathcal{F}\in \{S, S', L, E\}$ will be represented as $\mu^{\mathcal{F}}$ and the mean of the distribution $G$ as $\mu^G$. Similarly, $M_{\mathcal{S}}$ will denote matrices of the form $\mathbb{E}_{X\in_u\mathcal{S}}[(X-\mu^G)(X-\mu^G)^T]$ for $\mathcal{S}\in \{S, S', L, E\}$. DKKLMS provide two observations, which we rephrase as lemmas to set up our own proofs.

\begin{definition}\label{3.1}
    Given finite multisets $S$ and $S'$ we let $\Delta(S, S')$ denote the size of the symmetric difference between $S$ and $S'$ divided by the cardinality of $S$.
\end{definition}

\begin{lemma}\label{c1}
    $(|L| + |E|)/|S| = \Delta(S, S') \leq 2\gamma$
\end{lemma}

\begin{lemma}\label{c2}
    $(1 - 2\gamma)|S| \leq |S'| \leq (1+2\gamma)|S|$
\end{lemma}

Lemma \ref{c1} follows from the definition of $S'$ and the fact that it is a multiset that has been filtered from the original data set, which itself was at most $\gamma$-corrupted with respect to $S$. Lemma \ref{c2} is a direct algebraic result of Lemma \ref{c1}.

\subsection{Proof for Lemma \ref{lemma1}}

\begin{proof}
From condition (3) of Definition \ref{gooddef} on $(\gamma, \tau)$-good multisets, $||\mu^S - \mu^G||_2 \leq \gamma$. Furthermore, Lemma \ref{c2} provides that $1-2\gamma \leq |S'|/|S| \leq 1 + 2\gamma$, so $(|S|/|S'|)||\mu^S-\mu^G||_2 \leq\frac{\gamma}{1-2\gamma}$ as desired.
\end{proof}

\subsection{Proof for Lemma \ref{lemma2}}
\begin{proof}
Lemma 5.9 of \cite{hightractable} yields:
\begin{align*}|v^TM_L v|=\underset{X\in_u L}{\mathbb{E}} \left[|v\cdot (X-\mu^G)|^2\right] &\ll \log(|S|/|L|) + \gamma\cdot |S|/|L| \end{align*}
\noindent This implies:
\begin{align*}||\mu^L-\mu^G||_2 \leq \sqrt{||M_L||_2} &< \sqrt{\log(|S|/|L|) + \gamma\cdot |S|/|L|} \\&\leq \sqrt{\log(|S|/|L|)} + \sqrt{\gamma |S|/|L|}\end{align*}

Lemma \ref{c1} implies both $|L|/|S| \leq 2\gamma$ and $|E|/|S| \leq 2\gamma$. Furthermore, Lemma \ref{c2} implies $(1-2\gamma)|S| \leq |S'| \leq (1+2\gamma)|S|$. This allows for the bound to be derived as follows:
\begin{align*}
	\dfrac{|L|}{|S'|}||\mu^L -\mu^G||_2 &< \dfrac{|L|}{(1-2\gamma)|S|}\sqrt{\log(|S|/|L|)} + \dfrac{|L|\sqrt{\gamma |S|/|L|}}{(1-2\gamma)|S|}\\
	&\leq \frac{1}{1-2\gamma}\left[\dfrac{|L|}{|S|}\sqrt{\log(|S|/|L|)}+\sqrt{2\gamma^2}\right]\\
	&= \frac{1}{1-2\gamma}\left[\sqrt{2}\gamma + \sqrt{\dfrac{\log(|S|/|L|)}{(|S|/|L|)^2}}\right]\\
	&\leq\frac{1}{1-2\gamma}\left[\sqrt{2}\gamma + \sqrt{\dfrac{|L|}{|S|}}\right]\\
	&\leq \dfrac{\sqrt{2}\gamma + \sqrt{2\gamma}}{1-2\gamma}
\end{align*}
\end{proof}

\noindent where the second line follows from the bound on $|L|/|S|$ and the fourth line follows from the fact that $\left|\dfrac{\log(x)}{x}\right| < 1$ $\forall x \geq 1$. Because we know $|S|\geq |L|$ by construction, we have $\dfrac{|S|}{|L|} \geq 1$. 

\subsection{Proof for Lemma \ref{lemma3}}

\begin{proof}
With $\Sigma$ denoting the empirical covariance matrix of $S'$, then by definition:
\[\Sigma - I = M_{S'} - I - (\mu^{S'}-\mu^G)(\mu^{S'}-\mu^G)^T\]
\[\Sigma - I + (\mu^{S'}-\mu^G)(\mu^{S'}-\mu^G)^T = M_{S'} - I\]
\[||\Sigma-I||_2 + ||\mu^{S'}-\mu^G||_2^2 \geq ||M_{S'}-I||_2\]

Using another bound from Corollary 5.10 of \cite{hightractable}:
\[||\Sigma - I||_2 + ||\mu^{S'}-\mu^G||_2^2 \geq ||M_{S'} - I||_2 \geq (|E|/|S'|)||M_E||_2\]
Lemmas \ref{lemma1} and \ref{lemma2} jointly can provide the bound:
\[||\mu^{S'} - \mu^G||_2 \leq \dfrac{|E|}{|S'|}||\mu^E-\mu^G||_2+\kappa\]
where $\kappa = \dfrac{\gamma}{1-2\gamma} + \dfrac{\sqrt{2}\gamma + \sqrt{2\gamma}}{1-2\gamma}$ by triangle inequality on the two constituent bounds. Then using the fact that $||M_E||_2 \geq ||\mu^E - \mu^G||_2$:
\begin{align*}(|E|/|S'|)||M_E||_2 &\leq ||\Sigma - I||_2 + \left(\dfrac{|E|}{|S'|}||\mu^E-\mu^G||_2+\kappa\right)^2 \\&\leq ||\Sigma - I||_2 + \left(\dfrac{|E|}{|S'|}\sqrt{||M_E||_2}+\kappa\right)^2 \end{align*}
Here we let $a = |E|/|S'|$ and $u = ||\mu^E - \mu^G||_2$ to simplify notation.
\begin{align*}
    a u^2 &\leq ||\Sigma - I||_2 + (au + \kappa)^2\\
    (a-a^2)u^2-(2a\kappa)u &\leq \lambda^* + \kappa^2\\
    u^2-\dfrac{2\kappa u}{1-a} &\leq \dfrac{\lambda^*+\kappa^2}{a-a^2}\\
    u^2-2\left(\dfrac{\kappa}{1-a}\right)u +\left(\dfrac{\kappa}{1-a}\right)^2 &\leq \dfrac{\lambda^*+\kappa^2}{a-a^2}+\left(\dfrac{\kappa}{1-a}\right)^2\\
    \left(u-\dfrac{\kappa}{1-\alpha}\right)^2 &\leq \dfrac{\lambda^*+\kappa^2}{a-a^2}+\left(\dfrac{\kappa}{1-a}\right)^2\\
    \left(u - \dfrac{\kappa}{1-a}\right) &\leq \sqrt{\dfrac{\lambda^*+\kappa^2}{a-a^2}+\left(\dfrac{\kappa}{1-a}\right)^2}\\
    u &\leq \sqrt{\dfrac{\lambda^*+\kappa^2}{a-a^2}}+\dfrac{2\kappa}{1-a}
\end{align*}
\[\therefore (\mu^E - \mu^G) \leq \sqrt{\dfrac{\lambda^*+\kappa^2}{a-a^2}}+4\kappa\]

\noindent where the last line makes use of the fact that $a = |E|/|S| \leq 0.5$ by construction (because if more than $0.5|S|$ terms are drawn from a separate Gaussian, we can no longer produce a robust mean for the initial distribution) and then that $1/(1-x) \leq 2$ for $0 \leq x \leq 0.5$, which we will apply again below. The term of interest we care about is $au$, or $(|E|/|S|')||\mu^E-\mu^G||_2$.
\begin{align*}
    \dfrac{|E|}{|S|'}||\mu^E-\mu^G||_2 &\leq a\sqrt{\dfrac{\lambda^*+\kappa^2}{a-a^2}}+4a\kappa\\
    &\leq 2\kappa + \sqrt{a(\lambda^*+\kappa^2)}\sqrt{\dfrac{a}{a-a^2}}\\
    &\leq 2\kappa + \sqrt{2\gamma(\lambda^*+\kappa^2)}\sqrt{\dfrac{1}{1-a}}\\
    &\leq 2\kappa + \sqrt{4\gamma(\lambda^*+\kappa^2)}\\
    &\leq (2+2\sqrt{\gamma})\kappa + 2\sqrt{\gamma\lambda^*}
\end{align*}
\end{proof}

\end{document}